\def\year{2016}\relax

\documentclass[letterpaper]{article}
\usepackage{aaai16}
\usepackage{times}
\usepackage{helvet}
\usepackage{courier}

\usepackage{amssymb}
\usepackage{amsmath}
\usepackage{amsthm}
\usepackage[british]{babel}
\usepackage[linesnumbered,ruled,slide,boxed,vlined]{algorithm2e}

\usepackage{xifthen}
\usepackage{ifthen}
\usepackage{cases}

\usepackage{mdwtab}
\usepackage{url}

\newtheorem{definition}{Definition}

\newtheorem{theorem}{Theorem}
\newtheorem{proposition}{Proposition}
\newtheorem{example}{Example}

\newcommand\wrt{{\it w.r.t. }}

\begin{document}
%
\title{Query Answering with Inconsistent Existential Rules under \\Stable Model Semantics}
\author{Hai Wan{$^1$}, Heng Zhang{$^{2,}$\thanks{{Corresponding author.}}}, Peng Xiao{$^1$}, Haoran Huang{$^3$}, and Yan Zhang{$^4$}  \\
{$^1$}School of Data and Computer Science, Sun Yat-sen University, Guangzhou, China\\
\texttt{wanhai@mail.sysu.edu.cn}\\
{$^2$}School of Computer Science and Technology, Huazhong University of Science and Technology, Wuhan, China\\
\texttt{hengzhang@hust.edu.cn}\\
{$^3$}School of Computer Science, Fudan University, Shanghai, China\\
{$^4$}School of Computing, Engineering and Mathematics, Werstern Sydney University, Sydney, Australia
}

\maketitle
\begin{abstract}
\begin{quote}
\looseness=-1
Traditional inconsistency-tolerent query answering in ontology-based data access relies on selecting maximal components of an ABox/database which are consistent with the ontology.
However, some rules in ontologies might be unreliable if they are extracted from ontology learning or written by unskillful knowledge engineers. In this paper we present a framework of handling inconsistent existential rules under stable model semantics, which is defined by a notion called rule repairs to select maximal components of the existential rules.
Surprisingly, for R-acyclic existential rules with R-stratified or guarded existential rules with stratified negations,
both the data complexity and combined complexity of query answering under the rule {repair semantics} remain the same as that under the conventional query answering semantics.
This leads us to propose several approaches to handle the rule {repair semantics} by calling answer set programming solvers.
An experimental evaluation shows that these approaches have good scalability of query answering under rule repairs on realistic cases.
\end{quote}
\end{abstract}

\section{Introduction}
\vspace*{-0.3mm}
\looseness=-1
Querying inconsistent ontologies is an intriguing new problem that gives rise to a flourishing research activity in the description logic (DL) and existential rules community.
Consistent query answering, first developed for relational databases \cite{pods1999Arenas,icdt2007Chomic}
and then generalized as the AR and IAR semantics for several DLs \cite{rr2010LemboLRRS},
is the most widely recognized semantics for inconsistency-tolerant query answering.
These two traditional semantics are based upon the notion of {\em repair},
defined as an inclusion-maximal subset of the ABox consistent with the TBox.
\citeauthor{kais2013Du} \shortcite{kais2013Du} studied query answering under weight-based AR semantics for DL $\mathcal{SHIQ}$.
\citeauthor{aaai14BienvenuBG} \shortcite{aaai14BienvenuBG} studied variants of AR and IAR semantics for DL-Lite$_R$ obtained by replacing classical repairs with various preferred repairs.
Existential rules (also known as Datalog$^{\pm}$) are
set to play a central role in the context of query answering and information extraction for the Semantic Web.
\citeauthor{aaai2015Lukasiewicz} \shortcite{ecai2012Lukasiewicz,otm2013Lukasiewicz,aaai2015Lukasiewicz} studied the data complexity and combined complexity of AR semantics under the main decidable classes of existential rules enriched with negative constraints.

\looseness=-1
However, observe that some rules might be unreliable if they are extracted from ontology learning or written by unskillful knowledge engineer \cite{ws2011Lehmann}.
\citeauthor{MeyerLBP06} \shortcite{MeyerLBP06} proposed a tableau-like
algorithm which yields \textsc{ExpTime} as upper
bound for finding maximally concept-satisfiable terminologies
represented in $\mathcal{ALC}$.
\citeauthor{KalyanpurPSG06} \shortcite{KalyanpurPSG06} provided solutions on repairing unsatisfiable concepts in a consistent OWL ontology.
Furthermore, usually there exist preferences between rules, and
rules with negation are often considered less preferred than rules without negation.
\citeauthor{dlog2010Scharrenbach} \shortcite{dlog2010Scharrenbach}
proposed that the original axioms must be preserved in the knowledge
base under certain conditions and requires changing the underlying logics for repair.
\citeauthor{dlog2014Wang} \shortcite{dlog2014Wang}
proposed that when
new facts are added that contradict to the ontology,
it is often desirable
to revise the ontology according to the added data.
Therefore, this motivates us to consider another repair that selects maximal components of the existential rules. We illustrate the motivation via the following example.
\vspace*{-0.2mm}
\begin{example}\label{example:1}
Let $D= \{Bat(a), Mammal(a)\}$ be a database and let $\Sigma$ be the following rule set expressing that
each bat can fly and has at least one cave to live in;
and if one creature lives in cave then it is a trogloxene;
and if we do not know one mammal can fly then it can not fly;
if one creature can fly then it is a bird;
additionally a bird can not be a trogloxene at the same time;
similarly a bird can not be a mammal meanwhile.
\begin{equation*}
\footnotesize
\begin{aligned}
\!\!Bat(x) &\; \!\!\rightarrow \;\!\!CanFly(x),                       & \!\!\!\!\!(1)\\
\!\!Bat(x) &\; \!\!\rightarrow \;\!\!\exists y LiveIn(x, y), Cave(y),& \!\!\!\!\!(2)\\
\!\!LiveIn(x, y), Cave(y) &\; \!\!\rightarrow \;\!\!Trogloxene(x),    & \!\!\!\!\!(3)\\
\!\!Mammal(x), \mathsf{not} \;CanFly(x) &\; \!\!\rightarrow \;\!\!CanNotFly(x),& \!\!\!\!\!(4)\\
\!\!CanFly(x) &\; \!\!\rightarrow \;\!\!Bird(x),                      & \!\!\!\!\!(5)\\
\!\!Bird(x), Trogloxene(x) &\; \!\!\rightarrow \;\!\!\bot,            & \!\!\!\!\!(6)\\
\!\!Bird(x), Mammal(x) &\; \!\!\rightarrow \;\!\!\bot.                & \!\!\!\!\!(7)\\
\end{aligned}
\end{equation*}
Clearly $\langle\Sigma,D\rangle$ is inconsistent under stable model semantics.
We assume
$P_1 = \{(1),(2),(3)\}$ is more reliable (or preferred) than $P_2 = \{(4),(5),(6),(7)\}$ .
Then we can delete $(6)$ and $(7)$, or $(5)$ in $P_2$ to restore the consistency,
and get inclusion-maximal preferred consistent rule sets \wrt $D$:{\smallskip\\
\begin{footnotesize}
\hspace*{5mm} $\{(1),(2),(3),(4),(6),(7)\}$,~~~~~~~~~$\{(1),(2),(3),(4),(5)\}.$
\end{footnotesize}}
\end{example}

\looseness=-1
We will focus on the case where the database is reliable but rules are not.
Our main goal is to
present a framework of handling inconsistent existential rules under stable model semantics.
We define a notion called \textit{rule repairs} to select maximal components of the rules,
the philosophy behind that is to trust the rules as many as possible.
Our second goal is to perform an in-depth analysis
of the data and combined complexity of
inconsistency-tolerant query answering
under rule {repair semantics}.
Let us recall some previous work on existential rules under stable model semantics.
\citeauthor{ijcai2013Magka} \shortcite{ijcai2013Magka}
presented
R-acyclic and R-stratified normal rule sets each of which always admits at most one finite stable models.
\citeauthor{aaai2015heng} \shortcite{aaai2015heng} implicitly showed that the R-acyclicity is enough to capture all negation-free  rule sets with finite stable models.
\citeauthor{kr2014Gottlob} \shortcite{kr2014Gottlob}
proved the decidability of query answering under stable model semantics for guarded existential rules.
\citeauthor{AlvianoP15} \shortcite{AlvianoP15}
extended the stickiness notion to normal rule sets and
showed that it assures the decidability for well-founded semantics rather than stable model semantics.
We will focus on R-acyclic rule sets with R-stratified or full negations and guarded existential rules with stratified or full negations.

\looseness=-1
Our main contributions are briefly summarized as follows.
We define rule {repair semantics} to handle inconsistent existential rules under stable model semantics.
We consider rule repairs \wrt
inclusion-maximal subset or
cardinality,
and that with preference.
We obtain a (nearly) complete picture of the data and combined complexity of inconsistency-tolerant query answering under rule {repair semantics} (Table 1).
Surprisingly, for R-acyclic existential rules with R-stratified or guarded existential rules with stratified negations,
both the data complexity and combined complexity of query answering under the rule {repair semantics} remain the same as that under the conventional query answering semantics.
Interestingly,
the data complexity based upon weak-acyclic or guarded existential rules with stratified negation is \textsc{PTime}-complete.
This leads us to propose several approaches to handle the rule {repair semantics} by calling answer set programming (ASP) solvers.
An experimental evaluation shows that these approaches have good scalability of query answering rule repairs on realistic cases.

\section{Preliminaries}
We consider a standard first-order language.
We use $\mathrm{Var(\varepsilon)}$ to denote
the variables appearing in an expression $\mathrm{\varepsilon}$.

\vspace*{-5mm}
\paragraph{Databases.}
We assume an infinite set $\Delta$ of {\em (data) constants},
an infinite set $\Delta_n$ of {\em (labeled) nulls} (used as fresh Skolem terms),
and an infinite set $\Delta_v$ of {\em variables}.
A \textit{term} $t$ is a constant, a null, or a variable.
We denote by $\mathbf{x}$ a sequence of variables $x_1, \dots, x_k$ with $k\ge 0$.
An \textit{atom} $\alpha$ has the form $R(t_1, \dots, t_n)$,
where $R$ is an $n$-ary relation symbol, and $t_1, \dots, t_n$ are terms.
A conjunction of atoms is often identified with the set of all its atoms.
We assume a \textit{relational schema} $\mathcal{R}$, which is a finite set of relation
symbols.
An \textit{instance} $I$ is a (possibly infinite) set of facts $p(\mathbf{t})$, i.e., atoms without involving variables,
where $\mathbf{t}$ is a tuple of constants and nulls.
A \textit{database} $D$ over a relational schema $\mathcal{R}$ is a finite instance with relation symbols from $\mathcal{R}$
and with arguments only from $\Delta$ (i.e., without involving nulls).

\vspace*{-3mm}
\paragraph{Normal Logic Programs and Stable Models.}

Each {\em normal (logic) program} is a finite set of {\em NLP rules} of the form
\begin{equation}
\alpha\leftarrow\beta_1,\dots,\beta_n,\,\mathsf{not}\,\beta_{n+1},\dots,\,\mathsf{not}\,\beta_{m}
\end{equation}
where $\alpha,\beta_1,\dots,\beta_m$ are atoms and $m\ge n\ge0$. Given a rule $r$ of the above form, let $head(r)=\alpha$, let $body^+(r)=\{\beta_1,\dots,\beta_n\}$, and let $body^-(r)=\{\beta_{n+1},\dots,\beta_m\}$.

Let $\Pi$ be a normal program. The {\em Herbrand universe} and {\em Herbrand base} of $\Pi$ are denoted by $HU(\Pi)$ and $HB(\Pi)$, respectively. A variable-free rule $r'$ is called an {\em instance} of some rule $r\in\Pi$ if there is a substitution $\theta:\Delta_v\rightarrow HU(\Pi)$ such that $r\theta=r'$. Let $ground(\Pi)$, the {\em grounding} of $\Pi$, be the set of all instances of $r$ for all $r\in\Pi$.

The {\em Gelfond-Lifschitz reduct} of a normal program $\Pi$ \wrt a set $M\subseteq HB(\Pi)$, denoted $\Pi^M$, is the (possibly infinite) ground positive program obtained from $ground(\Pi)$ by
\begin{itemize}
\item deleting every rule $r$ such that $body^-(r)\cap M\neq\emptyset$, and
\item deleting all negative literals from each remaining rule.
\end{itemize}
\looseness=-1
A subset $M$ of $HB(\Pi)$ is called a {\em stable model} of $\Pi$ if it is the least model of $ground(\Pi^M)$. For more about stable model semantics, refer to \cite{GelfondL88,FerrarisLL11}.

\vspace*{-3mm}
\paragraph{Normal Existential Rules.}
Every {\em normal (existential) rule} is a first-order sentence of the form $\forall\mathbf{x}\forall\mathbf{y}\varphi(\mathbf{x},\mathbf{y})\rightarrow\exists\mathbf{z}\psi(\mathbf{x},\mathbf{z})$, where $\varphi$ is a conjunction of {\em literals}{, i.e., atoms or negated atoms (of the form $\neg\alpha$ where $\alpha$ is atomic),} $\psi$ is a conjunction of atoms, and each universally quantified variable appears in at least one positive conjunct of $\varphi$. In the above normal rule, $\varphi$ is called its {\em body}, and $\psi$ its {\em head}.
A normal rule is called a {\em constraint} if its head is the ``false" $\bot$.
For simplicity, when writing a rule, we often omit the universal quantifiers; by a {\em normal rule set}, we always mean a finite number of normal existential rules.

Let $r$ be a normal rule $\varphi(\mathbf{x},\mathbf{y})\rightarrow\exists\mathbf{z}\psi(\mathbf{x},\mathbf{z})$. For each variable $z\in\mathbf{z}$, we introduce an $n$-ary fresh function symbol $f^r_z$ where $n=|\mathbf{x}|$. The {\em skolemization} of $r$, denoted $\mathsf{sk}(r)$, is the rule obtained from $r$ by substituting $f^r_z(\mathbf{x})$ for $z\in\mathbf{z}$, followed by substituting ``$\mathsf{not}$" for $\neg$. Let $\Sigma$ be a normal rule set. We define $\mathsf{sk}(\Sigma)$ to be the set of rules $\mathsf{sk}(r)$ for all $r\in\Sigma$. Clearly, $\mathsf{sk}(\Sigma)$ can be regarded as a normal program in an obvious way. Given any database $D$, an instance is called a {\em stable model} of $D\cup\Sigma$ if it is a stable model of $D\cup\mathsf{sk}(\Sigma)$.

A normal rule $r$ is called {\em guarded} if there is a positive conjunct in the body of $r$ that contains all the universally quantified variable of $r$, and a normal rule set is called {\em guarded} if every rule in it is guarded.

A normal rule set $\Sigma$ is {\em stratified} if there is a function $\ell$ that maps relation symbols to  integers such that for all $r\in\Sigma$:
\begin{itemize}
\item for all relation symbols $R$ occurring in the head and $S$ positively occurring in the body, $\ell(R)\ge\ell(S)$, and
\item for all relation symbols $R$ occurring in the head and $S$ negatively occurring in the body, $\ell(R)>\ell(S)$.
\end{itemize}
Sometimes, the negations that occur in a stratified normal rule set are called {\em stratified negations}, and those in a non-stratified normal rule set are called {\em full negations}.

Let $r_1$ and $r_2$ be two normal rules, and let $B_i^+$ (resp., $B_i^-$ and $H_i$)  be the set of atoms positively (resp., negatively and positively) occurring  in the body (resp., body and head) of $r_i$.
{\it W.l.o.g.},
assume that no variable occurs in both $r_1$ and $r_2$.
Rule $r_2$ {\em positively
relies} on $r_1$, written $r_1 \rightarrow^+ r_2$, if there exist a database $D$ and a substitution $\theta$ such that
$B^+_1\theta \subseteq D$, $B^-_1\theta \cap D = \emptyset$, $B^+_2\theta \subseteq D\cup H_1\theta$,
$B^-_2\theta \cap (D\cup H_1\theta) = \emptyset$,
$B^+_2\theta \nsubseteq D$ and
$H_2\theta \nsubseteq D\cup H_1\theta$.
Rule $r_2$ {\em negatively
relies} on $r_1$, written $r_1 \rightarrow^- r_2$, if there exist a database $D$
and a substitution $\theta$ such that
$B^+_1\theta \subseteq D$,
$B^-_1\theta \cap D = \emptyset$,
$B^+_2\theta \subseteq D$,
$B^-_2\theta \cap H_1\theta \neq \emptyset$ and
$B^-_2\theta \cap D = \emptyset$.
A normal rule set $P$ is called {\em R-acyclic} if there is no cycle of positive reliances
$r_1 \rightarrow^+ \ldots \rightarrow^+ r_n \rightarrow^+ r_1$
that involves
a rule with an existential quantifier,
and $P$ is called {\em R-stratified} if there is a partition
$\{P_1, \ldots, P_n\}$ of $P$ such that, for every two normal rule sets $P_i, P_j$ and rules
$r_1\in P_i$ and $r_2\in P_j$,
if $r_1 \rightarrow^+ r_2$ then $i \leq j$ and if $r_1 \rightarrow^- r_2$ then $i < j$.

\vspace*{-4mm}
\paragraph{Classical Boolean Query Answering.}
{A {\em normal Boolean conjunctive query (NBCQ)} $Q$ is an existentially closed conjunction of atoms and negated atoms involving no null. Let $Q^+$ (respectively., $Q^-$) be the set of atoms positively (respectively., negatively) occurring in $Q$.} An NBCQ is called {\em safe} if every variable in an atom from $Q^-$ has at least one occurrence in $Q^+$; it is {\em covered} if for every atom $\alpha$ in $Q^-$, there is an atom in $Q^+$ that contains all arguments of $\alpha$.

Given a database $D$ and an NBCQ $Q$, we write $D\models Q$ if there exists an assignment $h$ (that is, a function that maps each variable to a variable-free term) such that $h(Q^+)\subseteq D$ and $h(Q^-)\cap D=\emptyset$. Furthermore, given a database $D$, a normal rule set $\Sigma$ and an NBCQ $Q$, we write $D\cup\Sigma\models_s Q$ if, for each stable model $M$ of $D\cup\Sigma$, we have that $M\models Q$.

\vspace*{-4.6mm}
\paragraph{Complexity Classes.}
We assume that the reader is familiar with the complexity theory. Given a unary function $T$ on natural numbers, by $\textsc{DTime}(T(n))$ ($\textsc{NTime}(T(n))$, respectively) we mean the class of languages decidable in time $T(n)$ by a deterministic (nondeterministic, respectively) Turing machine. Besides the well-known complexity classes such as $(co)(\textsc{N})\textsc{PTime}$ and $(co)(\textsc{N})\textsc{2ExpTime}$,
we will also use several unusual classes as follows. By notation
$\Delta_2$-$\textsc{2ExpTime}$ we mean the class of all languages decidable in exponential time by a deterministic Turing machine with an oracle for some $\textsc{N2ExpTime}$-complete problem.  The Boolean hierarchy ($\textsc{BH}$) is defined as follows: $\textsc{BH}(1)$ is $\textsc{NPTime}$; for $k\ge 1$, $\textsc{BH}(2k)$ ($\textsc{BH}(2k+1)$) is the class of languages each of which is the intersection (union, respectively) of a language in $\textsc{BH}(2k-1)$ ($\textsc{BH}(2k)$, respectively) and a language in $co\textsc{NPTime}$ ($\textsc{NPTime}$, respectively); $\textsc{BH}$ is then the union of $\textsc{BH}(n)$ for all $n\ge 1$. Note that $\textsc{DP}$, the class for {\em difference polynomial time}, is exactly the class $\textsc{BH}(2)$; $\textsc{BH}(2k)$ is actually the class of languages each of which is the union of $k$ languages in $\textsc{DP}$; and $\textsc{BH}$ is closed under complement. It was shown by~\cite{siamcomp1996Chang96} that a collapse of the Boolean hierarchy  implies a collapse of the polynomial hierarchy; thus it seems impossible to find a $\textsc{BH}$-complete problem.

\vspace*{-0.5mm}
\section{Existential Rule {Repair Semantics}}
\vspace*{-0.5mm}
\looseness=-1
In this section, we propose several semantics to handle inconsistency in ontological knowledge base.
Different from many existing works, we will focus on the case where the database is reliable but rules are not. Similar to the data {repair semantics}, see \cite{rr2010LemboLRRS}, our inconsistency-tolerant semantics will rely on a notion called {\em rule repairs}.

To define rule repairs, we arm every rule set with a preference. Such rule sets are called preference-based ontologies.
\vspace*{-3mm}
\begin{definition}
Each {\em preference-based ontology} is an ordered pair $(\Sigma,\preceq)$, where $\Sigma$ is a normal rule set, and $\preceq$ is a preorder (i.e., a reflexive and transitive binary relation) on $\mathcal{P}(\Sigma)$ (i.e., the power set of $\Sigma$). We call $\preceq$ a {\em preference}.
\end{definition}

Now, we are in the position to define rule repairs.
\vspace*{-1mm}
\begin{definition}
Let $O$ be a preference-based ontology $(\Sigma,\preceq)$ and $D$ a database. A subset $S$ of $\Sigma$ is called a {\em (preferred rule) repair of $\Sigma$ \wrt $\preceq$ and $D$} (or simply a repair \wrt \!$\preceq$  if $\Sigma$ and $D$ are clear from the context) if $D\cup S$ has at least one stable model, and for all subsets $S'$ of $\Sigma$ with $S\prec S'$ (i.e., $S\preceq S'$ but $S'\not\preceq S$), $D\cup S'$ has no stable model.
\end{definition}

\looseness=-1
Intuitively, {a preferred rule repair is a maximal component of the rule set which is consistent with the current database.} The philosophy behind it is to trust the rules as many as possible. Note that the number of repairs are normally more than one. To avoid a choice among them, we follow the spirit of ``certain" query answering. The semantics is then as follows.

\begin{definition}
Let $O$ be a preference-based ontology $(\Sigma,\preceq)$ where $\Sigma$ is a normal rule set, and let $D$ be a database and $Q$ an NBCQ. Then we write $\langle D, O\rangle\models Q$ if, for all preferred rule repairs $S$ of $\Sigma$ \wrt $\preceq$ and $D$, we have $D\cup S\models_s Q$.
\end{definition}

The following proposition shows us that our semantics for inconsistency-tolerant query answering will  coincide with the classical semantics for query answering if the ontological knowledge base is consistent, which is clearly important.

\begin{proposition}
Let $O$ be a preference-based ontology $(\Sigma,\!\preceq)$ and let $D$ be a database.
If $\Sigma \cup D$ has a stable model, then $\langle D,O\rangle \models Q$ iff $\Sigma \cup D \models_s Q$ for any NBCQ $Q$.
\end{proposition}

With the above definitions, we then have a framework to define semantics for rule-based inconsistency-tolerant query answering. To define concrete semantics, we need to find  preferences which will be useful in real-world applications.
Besides the preference based on the set inclusion $\subseteq$, similar to \cite{aaai14BienvenuBG},
we will consider other four kinds of preferences over subsets, which were first proposed by \cite{jacm1995Eiter} to study logic-based abduction.

\vspace*{-3.6mm}
\paragraph{Cardinality ($\leq$).}
Given any $S,S'\subseteq\Sigma$, we write $S \leq S'$ if $|S|\leq|S'|$.
The intuition of using this preference is that we always prefer the rule set with the maximum number of rules which are most likely to be correct.

\vspace*{-3.6mm}
\paragraph{Priority Levels ($\subseteq_{P}$, $\leq_{P}$).} Every {\em prioritization} $P$ of $\Sigma$ is a tuple $\langle P_1,\dots,P_n\rangle$ where $\{P_1,\dots,P_n\}$ is a partition of $\Sigma$. Given a {\em prioritization} $P=\langle
P_1,\dots,P_n\rangle$ of $\Sigma$, the preferences $\subseteq_{P}$ and $\leq_{P}$ can be
defined as follows:
\begin{itemize}
  \item Prioritized set inclusion ($\subseteq_P$): Given $S,S'\subseteq\Sigma$, we write $S \subseteq_{P} S'$ if $S \cap P_i=S' \cap P_i$ for every $1\leq i\leq n$, or there is some $1\leq i\leq n$ such that $S \cap P_i \subsetneq S' \cap P_i$ and for all $1 \leq j < i$, $S \cap P_j = S' \cap P_j$.
  \item Prioritized cardinality ($\le_P$): Given $S,S'\subseteq\Sigma$, we write $S \le_{P} S'$ if $|S  \cap P_i| = |S' \cap P_i|$ for every $1\leq i\leq n$, or there is some $1\leq i\leq n$ such that
      $|S \cap P_i|$ $<$ $|S' \cap P_i|$ and for all $1 \leq j < i$, $|S \cap P_j| = |S' \cap P_j|$.
\end{itemize}
\vspace*{-4.8mm}

\paragraph{Weights ($\leq_{w}$).} A {\em weight assignment} is a function
$w: \Sigma \rightarrow \mathbb{N}$. Given two sets $S,S'\subseteq\Sigma$ and a weight assignment $w$, we write
$S \leq_{w} S'$ if $\sum_{r\in S}w(r) \leq
\sum_{r\in S'}w(r)$.

\medskip
In the rest of this paper, we will fix $P$ as a prioritization and $w$ as a weight assignment  unless otherwise noted.

\vspace*{-1mm}
\begin{example} [Example \ref{example:1} continued] Let $\Sigma$ and $D$ be the same as in Example \ref{example:1}. Then
the repairs \wrt \!$\subseteq$ and $D$ are:{\smallskip\\
\begin{footnotesize}
\hspace*{5mm} $\{(1),(3),(4),(5),(6)\}$, ~~~~~~~~~$\{(1),(2),(3),(4),(5)\},$ \\
\hspace*{5mm} $\{(1),(2),(4),(5),(6)\}$, ~~~~~~~~~$\{(1),(2),(3),(4),(6),(7)\},$ \\
\hspace*{5mm} $\{(2),(3),(4),(5),(6),(7)\}.$
\end{footnotesize}\smallskip\\
The repairs \wrt \!$\leq$ and $D$ include:\smallskip\\
\begin{footnotesize}
\hspace*{5mm}  $\{(1),(2),(3),(4),(6),(7)\}$, ~~~$\{(2),(3),(4),(5),(6),(7)\}$.
\end{footnotesize}\smallskip\\
Let $P=\langle P_1,P_2 \rangle$ where $P_1,P_2$ are the same as in Example \ref{example:1}. Then the repairs \wrt \!$\subseteq_P$ and $D$ are shown in Example \ref{example:1}, and the repairs \wrt $\leq_{P}$ and $D$ are:\smallskip\\
\begin{footnotesize}
\hspace*{5mm}  $\{(1),(2),(3),(4),(6),(7)\}$.
\end{footnotesize}\smallskip\\
Let $w$ be  the weight assignment that maps each rule to its index. Then the only repair \wrt $\leq_{w}$ and $D$ is:\smallskip\\
\noindent\begin{footnotesize}
\hspace*{5mm}  $\{(2),(3),(4),(5),(6),(7)\}$.
\end{footnotesize}\smallskip\\}
Let $Q_a$ be query ``\textit{Mammal(a)}" and $Q_b$ be query ``\textit{Bird(a)}", then we have $\langle D,(\Sigma,\subseteq)\rangle\models Q_a$ and $\langle D,(\Sigma,\subseteq_P)\rangle\models Q_a$, but $\langle D,(\Sigma,\subseteq)\rangle\not\models Q_b$ and $\langle D,(\Sigma,\subseteq_P)\rangle\not\models Q_b$.
\end{example}

We find that repairs under $\subseteq_P$, $\leq$, $\leq_P$, and $\leq_w$ are the subset of the inclusion-maximal repairs.
\begin{theorem}
The repairs under $\subseteq_P$, $\leq$, $\leq_P$, $\leq_w$ are the subset of the repairs under $\subseteq$.
\end{theorem}
\begin{proof}
Let $S$ be the set of repairs under $\subseteq$, $S_P$ be the set of repairs under $\subseteq_P$, we prove that $S_P \subseteq S$.
Suppose for contradiction that $S_P \not\subseteq S$, then there exists a repair $R$, $R \in S_P$ and $R \not\in S$. Because the repairs in $S$ are inclusion-maximal, we have $R \subset R'$ for some $R'\in S$.
It is clear that $R \subset_P R'$, then $R$ is not a $\subseteq_P$ repair which contradict our assumption.
	
The rest semantics can be proved similarly.
\end{proof}

\section{Complexity Results}
In this section, we study the data and combined complexity for query entailment under
our rule {repair semantics}. In particular, we focus on the following decision problems:
\vspace*{-0.5mm}
\begin{itemize}
\item {\bf Data complexity}: Fixing a preference-based ontology $O$ and an NBCQ $Q$, given any database $D$ as input, deciding whether $\langle D, O\rangle\models Q$.
\item  {\bf Combined complexity}: Given any preference-based ontology $O$, any NBCQ $Q$ and any database $D$ as input, deciding whether $\langle D, O\rangle\models Q$.
\end{itemize}
\vspace*{-1.5mm}
To measure the size of input, we fix a natural way to represent a database $D$, a normal rule set $\Sigma$, an NBCQ $Q$, a prioritization $P$ and a weight assigning function $w$, and let $|\!|D|\!|, |\!|\Sigma|\!|, |\!|Q|\!|, |\!|P|\!|,|\!|w|\!|$ denote the sizes of $D,\Sigma,Q,P,w$, respectively, \wrt the fixed representing approach. Given a preference-based ontology $O=(\Sigma,\preceq)$, we define
$$
|\!|O|\!|:=\left\{
\begin{aligned}
&|\!|\Sigma|\!| &\text{ if }&\preceq\,\in\{\subseteq,\le\}, \\
&|\!|\Sigma|\!|+|\!|P|\!| &\text{ if }&\preceq\,\in\{\subseteq_P,\le_P\}, \\
&|\!|\Sigma|\!|+|\!|w|\!| & \text{ if }&\preceq\,=\,\le_w.
\end{aligned}
\right.
$$
By properly representing, we can have that $|\!|O|\!|=|\!|\Sigma|\!|^{\mathcal{O}(1)}$.

The following result is obvious.
\begin{proposition}\label{prop:cmplx_compare}
\hspace{-.1cm}Let $O$ be a preference-based ontology {$(\Sigma,\!\preceq)$}, where $\preceq\,\in\{\subseteq,\leq,\leq_{P},\subseteq_P,\leq_w\}$. Then, given any subsets $S,S'\subseteq\Sigma$, deciding whether $S\prec S'$ is in $\textsc{DTime}(|\!|O|\!|^{\mathcal{O}(1)})$.
\end{proposition}

Now, let us consider the complexity of query answering for R-acyclic and R-stratified rule sets under our semantics.
\vspace*{-6mm}
\begin{algorithm}
{\small
\SetKwInOut{KwIn}{Input}
\SetKwInOut{KwOut}{Output}
\caption{$\mathsf{PRQA}(D,O,Q)$}\label{alg:PQA}
\KwIn{a database $D$, a preference-based ontology $O=(\Sigma,\preceq)$, and a Boolean query $Q$}
\KwOut{{\em true} if $\langle D, O\rangle \models Q$, and {\em false} otherwise}
\ForEach{$S \subseteq \Sigma$}
{
    \If{$D\cup S$ has at least one stable model}
    {
     $isRepair :=$ {\em true}\;
    \ForEach{$S'\subseteq\Sigma$ with $S \prec S'$}
    {
        \If{$D\cup S'$ has at least one stable model}
        {
            $isRepair :=$ {\em false}\;
            {\bf break}\;
        }
    }
    \If{$isRepair$ and $D\cup S \not \models_s Q$}
    {
        \Return { false}\;
    }
    }
}
\Return { true}\;
}
\vspace*{-1mm}
\end{algorithm}
\vspace*{-5mm}
\begin{theorem}\label{thm:cmplx_wa_stra}
Let $O$ be a preference-based ontology $(\Sigma,\preceq)$, where $\Sigma$ is R-acyclic and R-stratified, and $\preceq\,\in\{\subseteq,\leq,\subseteq_P,$ $\leq_{P},\leq_w\}$.
Given a database $D$ and a safe NBCQ $Q$, deciding whether $\langle D, O\rangle \models Q$ is
\textsc{PTime}-complete for data complexity, and $2\textsc{ExpTime}$-complete for combined complexity.
\end{theorem}
\vspace*{-2mm}
\noindent \begin{proof}
Let $D$ be a database and $Q$ be a safe NBCQ.
By the definition of semantics, it is easy to verify that the problem of
deciding whether $\langle D, O\rangle \models Q$ can be solved by Alg.~\ref{alg:PQA}.

First, we consider the data complexity. In Alg.~\ref{alg:PQA}, let us fix a preference-based ontology $O=(\Sigma,\preceq)$ as defined in this theorem, fix a safe NBCQ $Q$, and let $D$ be the only input. As $\Sigma$ is R-acyclic and R-stratified, by Theorem 5 in~\cite{ijcai2013Magka}, it is clear that the body of the second loop (the inside one) in Alg.~\ref{alg:PQA} is computable in $\textsc{PTime}$ \wrt $D$. (Note that the existence of stable models can be reduced to the query answering problem in a routine way.) Since the second loop will be repeated a constant times, and by Proposition~\ref{prop:cmplx_compare} the loop condition can be checked in a constant time. (Note that the rule set $\Sigma$ is fixed now.) Thus, the second loop can be computed in $\textsc{PTime}$ \wrt the size of $D$. By a similar argument, we can show that Alg.~\ref{alg:PQA} can be implemented in $\textsc{PTime}$ \wrt $D$. This then completes the proof of membership. The hardness follows from the $\textsc{PTime}$-hardness of Datalog for data complexity, see, e.g.,~\cite{csur2001DantsinEGV}.

Next, we prove the combined complexity. Again, first address the membership. Let $n$ be the number of rules in $\Sigma$. Clearly, the body of the second loop will be repeated at most $2^n$ times. By Theorem 9 in~\cite{ijcai2013Magka}, it is computable in $\textsc{DTime}(2^{2^{|\!|\Sigma|\!|^{\mathcal{O}(1)}}})$. By Proposition~\ref{prop:cmplx_compare}, it is also clear that the loop condition can be checked in $\textsc{DTime}(|\!|O|\!|^{\mathcal{O}(1)})$. So, the second loop is computable in $\textsc{DTime}(2^{2^{|\!|O|\!|^{\mathcal{O}(1)}}})$ since $n\le |\!|\Sigma|\!|\le|\!|O|\!|$. By a similar evaluation, we know that the algorithm is implementable in $\textsc{DTime}(2^{2^{|\!|O|\!|^{\mathcal{O}(1)}}})$. Thus, the combined complexity is in $2\textsc{ExpTime}$. And the hardness follows from the $2\textsc{ExpTime}$-hardness of query answering of the R-acyclic language~\cite{ijcai2013Magka} and the fact that $D\cup\Sigma\models_s Q$ iff $\langle D,(\Sigma^\ast,\preceq)\rangle\models q$, where $\Sigma^\ast$ is $\Sigma\cup\{Q\rightarrow q\}$ and $q$ a fresh 0-ary relational symbol.
\end{proof}
\vspace*{-3mm}
\begin{theorem}\label{thm:cmplx_wa_full}
Let $O$ be a preference-based ontology $(\Sigma,\preceq)$, where $\Sigma$ is R-acyclic with full negations
and $\preceq\,\in\!\{\subseteq,\le,\subseteq_P,$ $\le_P,\le_w\}$.
Then, given a database $D$ and a safe
NBCQ $Q$, deciding whether $\langle D, O\rangle \models Q$ is in
\textsc{BH} for data complexity and in $\Delta_2$-$2\textsc{ExpTime}$ for combined complexity.
\end{theorem}{
\vspace*{-3mm}
\noindent \begin{proof}
We first prove the data complexity. To do this, we need to define some notations. Let $\mathcal{R}$ be the schema of $\Sigma$. Given any subset $X$ of $\Sigma$, let $L^X$ be the set of all $\mathcal{R}$-databases $D$ such that\vspace{-.1cm}
\begin{enumerate}
\item
$D\cup X$ has at least one stable model, and\vspace{-.1cm}
\item $D\cup X\models_{s} Q$ does not hold, and\vspace{-.1cm}
\item for all $Y\subseteq\Sigma$ with $X\prec Y$, $D\cup Y$ has no stable model.
\end{enumerate}\vspace{-.1cm}
Let $L$ denote the union of $L^X$ for all subsets $X$ of $\Sigma$. By the definition of the rule {repair semantics}, it is easy to see that $\langle D, O\rangle \models Q$ iff there is no $X\subseteq\Sigma$ such that $D\in L^X$, iff $D$ does not belong to $L$. Thus, if the following claim is true, by the definition of $\textsc{BH}$ we then have the desired result. Notice that the complexity class $\textsc{BH}$ is closed under complement.

\smallskip
{\noindent\em Claim.} Given any subset $X$ of $\Sigma$, it is in \textsc{DP} (\wrt the size of input database $D$) to determine whether $D\in L^X$.
\smallskip

Now, it remains to show the claim. Fix a subset $X\subseteq\Sigma$. Let $L_1$ denote the set of all $\mathcal{R}$-databases such that conditions 1 and 2 hold, and let $L_2$ denote the set of all $\mathcal{R}$-databases such that the condition 3 holds. According to Theorem 2 in~\cite{ijcai2013Magka}, $L_1$ is in $\textsc{NPTime}$ and $L_2$ in $co\textsc{NPTime}$. (Note that, as $\Sigma$ and $X$ are fixed, the number of subsets $Y$ is independent on the size of input database; thus $L_2$ should be in $co\textsc{NPTime}$.) By definition, $L^X=L_1\cap L_2$ is in $\textsc{DP}$. This proves the data complexity.

Next, we show the combined complexity. It is clear that $\langle D, O\rangle\models Q$ holds iff there does not exist $S\subseteq\Sigma$ such that\vspace{-.1cm}
\begin{enumerate}
\item $D\cup S$ has at least one stable model, and\vspace{-.1cm}
\item $D\cup S\models_s Q$ does not hold, and\vspace{-.1cm}
\item for all $S'\subseteq\Sigma$ with $S\prec S'$, $D\cup S'$ has no stable models.
\end{enumerate}\vspace{-.1cm}
By Theorem 2 in~\cite{ijcai2013Magka} and an analysis similar to that in Theorem~\ref{thm:cmplx_wa_stra} (for combined complexity), it is not difficult to see that, fixing $S\subseteq\Sigma$, both conditions 1 and 2 are in $co\textsc{N2ExpTime}$, and condition 3 is in $\textsc{N2ExpTime}$. For ``there does not exist $S\subseteq\Sigma$", we can simply enumerate all subsets $S$, which can be done in $2^{|\Sigma|}$ times. Therefore, query answering under the mentioned semantics must be in $\Delta_2$-$2\textsc{ExpTime}$ for combined complexity, which is as desired.
\end{proof}}
\vspace*{-2mm}

{Now let us focus on guarded rules. The proof of the following is similar to that of Theorem~\ref{thm:cmplx_wa_stra}, but employs the complexity results in~\cite{ws2012Cali}. The only thing we should be careful about is the constraints.}

\begin{theorem}\label{thm:cmplx_grd_stra}
Let $O$ be a preference-based ontology $(\Sigma,\preceq)$, where $\Sigma$ is guarded and stratified, and $\preceq\,\in\{\subseteq,\leq,\subseteq_P,\leq_{P},$ $\leq_w\}$.
Given a database $D$ and a covered NBCQ $Q$, deciding whether $\langle D, O\rangle \models Q$ is
\textsc{PTime}-complete for data complexity, and $2\textsc{ExpTime}$-complete for combined complexity.
\end{theorem}

For guarded rules with full negations, we have some results as below, where the proof for data complexity is similar to that in Theorem~\ref{thm:cmplx_wa_full}, and the proof for combined complexity is similar to that in Theorem~\ref{thm:cmplx_wa_stra}. Both results rely on the corresponding complexity results in~\cite{kr2014Gottlob}.

\begin{theorem}\label{thm:cmplx_grd_full}
Let $O$ be a preference-based ontology $(\Sigma,\preceq)$, where $\Sigma$ is guarded,
and $\preceq\,\in\{\subseteq,\le,\subseteq_P,\le_P,\le_w\}$.
Then, given a database $D$ and a covered
NBCQ $Q$, deciding whether $\langle D, O\rangle \models Q$ is in
\textsc{BH} for data complexity and $2\textsc{ExpTime}$-complete for combined complexity.
\end{theorem}

Finally, we conclude the results of this section as follows:
\vspace*{-5mm}
\begin{table}[!h]\label{table:1}
\setlength{\belowcaptionskip}{0.01pt}
\tabcolsep 0pt
\begin{center}
\begin{tabular}{c||c|c}
  & ~~\textbf{Data complexity}\, & ~~\textbf{Combined complexity}\, \\
  \hline
  \,RA + RS~~~\,      & ~\textsc{PTime}-complete    & ~\textsc{2ExpTime}-complete \\
  RA + Full~~   & ~in \textsc{BH} & ~in $\Delta_2$-$2\textsc{ExpTime}$ \\
  ~~~G + Stra~~    & ~\textsc{PTime}-complete    & ~\textsc{2ExpTime}-complete \\
  ~~~G + Full~~ & ~in \textsc{BH}            & ~\textsc{2ExpTime}-complete \\
  \hline
\end{tabular}
\end{center}
\vspace*{-3mm}
\caption{
The data and combined complexity of Boolean query answering over normal rule sets under preference-based semantics for 5 types
of preferred rule repairs, including $\subseteq$, $\leq$, $\subseteq_{P}$, $\leq_{P}$, and $\leq_{w}$.
Here, ``RA" means ``R-acyclic rule sets",
``G" means ``guarded rule sets",
 ``RS" means ``with R-stratified negations",
``Stra" means ``with stratified negations", and
``Full" means ``with full negations".
}
\end{table}
\vspace*{-7mm}

\section{Experimental Evaluation}
To demonstrate the effectiveness, we have implemented a prototype system for query answering of R-acyclic rule languages under the rule-{repair semantics} \wrt $\leq$, $\subseteq_{P}$, $\leq_{P}$ and $\leq_{w}$,
by calling a state-of-the-art ASP solver.
\vspace*{-2mm}
\subsection{From Query Answering to ASP}
To improve the efficiency,
we adopt particular algorithm for each rule-{repair semantics}.
The algorithms are all based on breadth-first search.
Finding rule repairs \wrt $\subseteq$ uses the basic process illustrated in Alg. \ref{alg:PQA},
and exponential checking will be conducted during the process.
For rule repairs \wrt $\leq$, though it works better than $\subseteq$ for the reason that there is no need to search the rest levels once it finds consistent sets.
As for rule repairs \wrt $\subseteq_{P}$,
we design an algorithm which iterates over the rules from low to high prioritization.
Once finding consistent results in the rules with lower prioritization, the searching stops.
It's known that $\leq_{P}$ can be translated into $\leq_{w}$, but not vice versa.
As for $\leq_{w}$, we search by deleting rules from the lowest weight to the greatest.

As a whole, the algorithms for situations with prioritization or weights will be much more efficient if the rule set satisfies the following two conditions:\vspace{-1mm}
\begin{itemize}
	\item The size of rules with lower prioritization (less weights) is very small, even though the whole rule set is large;\vspace{-1mm}
	\item The rule set can be consistent by only deleting rules with lower prioritization (less weights).\vspace{-1mm}
\end{itemize}
These conditions can be easily found in real applications because incorrectness are mostly caused by the rules newly added and the amount of these rules is normally small.

\vspace*{-2mm}
\subsection{Experiments}
We developed a prototype system {\sf QAIER}\footnote{http://ss.sysu.edu.cn/\%7ewh/qaier.html}
(Query Answering with Inconsistent Existential Rules)
in C++.
{\sf QAIER} can answer queries with inconsistent R-acyclic rule sets.
When it needs to check the existence of stable models,
{\sf QAIER} invokes an ASP solver {\sf clingo-4.4.0}\footnote{$\textrm{clingo-4.4.0}$.~http://sourceforge.net/projects/potassco/files/clingo/}.
\vspace*{-1.5mm}
\begin{table}[!h]
\setlength{\belowcaptionskip}{0.01pt}
\vspace*{-0.5mm}
\tabcolsep 0pt
\begin{center}
\tiny
\def\temptablewidth{1.02\columnwidth}
\begin{tabular}{|p{1.15cm}|m{1.1cm}|m{0.7cm}|m{1.1cm}|m{1.1cm}|m{1.1cm}|m{1.1cm}|m{1.1cm}|}
\hline
\cline{2-7}  {\sf ~Instance id} &
\multicolumn{1}{c|}{ {\sf \#facts}~} & \multicolumn{1}{c|}{ {\sf \#negs}~} & \multicolumn{1}{c|}{{\sf $t_\subseteq$}} &
\multicolumn{1}{c|}{{\sf $t_\leq$}~} & \multicolumn{1}{c|}{{\sf $t_{\subseteq_P}$}~} & \multicolumn{1}{c|}{{\sf $t_{\leq_P}$}} &
\multicolumn{1}{c|}{{\sf $t_{\leq_w}$}~}  \\
 \hline
 \hline
$~{\sf d6t3}$~&\raggedleft $6000~$ &\raggedleft  $9~$ &\raggedleft $1757.350~$ &\raggedleft $956.663~$ &\raggedleft $11.366~$ &\raggedleft $12.563~$  &\multicolumn{1}{r|}{$~~17.457~$}  \\
$~{\sf d6t5}$~&\raggedleft $6000~$ &\raggedleft  $11~$ &\centering --- &\raggedleft $968.864~$ &\raggedleft $19.073~$ &\raggedleft $32.445~$  &\multicolumn{1}{r|}{$~~47.449~$}  \\
$~{\sf d12t5}$~&\raggedleft $12000~$ &\raggedleft  $11~$ &\centering ---  &\raggedleft $1743.244~$ &\raggedleft $35.711~$ &\raggedleft $76.927~$  &\multicolumn{1}{r|}{$~~50.160~$}  \\
$~{\sf d30t5}$~&\raggedleft $30000~$ &\raggedleft  $11~$ &\centering ---  &\centering --- &\raggedleft $81.830~$ &\raggedleft $187.898~$  &\multicolumn{1}{r|}{$~~124.630~$}  \\
$~{\sf d110t5}$~&\raggedleft $110449~$ &\raggedleft  $11~$ &\centering --- &\centering --- &\raggedleft $365.412~$ &\raggedleft $267.529~$  &\multicolumn{1}{r|}{$~~149.574~$}  \\
$~{\sf d252t3}$~&\raggedleft $252498~$ &\raggedleft  $9~$ &\centering ---  &\centering --- &\raggedleft $278.426~$ &\raggedleft $466.643~$  &\multicolumn{1}{r|}{$~~147.217~$}  \\
$~{\sf d252t5}$~&\raggedleft $252498~$ &\raggedleft  $11~$ &\centering --- &\centering --- &\raggedleft $843.653~$ &\raggedleft $579.122~$  &\multicolumn{1}{r|}{$~~186.371~$}  \\
$~{\sf d500t3}$~&\raggedleft $500000~$ &\raggedleft  $9~$ &\centering ---  &\centering --- &\raggedleft $308.647~$ &\raggedleft $605.476~$  &\multicolumn{1}{r|}{$~~168.928~$}  \\
$~{\sf d500t5}$~&\raggedleft $500000~$ &\raggedleft  $11~$ &\centering --- &\centering --- &\raggedleft $1464.986~$ &\raggedleft $619.252~$ &\multicolumn{1}{r|}{$~~200.227~$}  \\
$~{\sf d686t3}$~&\raggedleft $686028~$ &\raggedleft  $9~$ &\centering ---  &\centering --- &\raggedleft $410.804~$ &\raggedleft $615.243~$  &\multicolumn{1}{r|}{$~~230.507~$}  \\
$~{\sf d686t5}$~&\raggedleft $686028~$ &\raggedleft  $11~$ &\centering --- &\centering --- &\centering --- &\centering ---  &\multicolumn{1}{r|}{$~~247.218~$}  \\
$~{\sf d123t3}$~&\raggedleft $1236999~$ &\raggedleft  $9~$ &\centering ---  &\centering --- &\centering --- &\raggedleft $727.710~$ &\multicolumn{1}{r|}{$~~345.231~$}  \\
$~{\sf d1236t5}$~&\raggedleft $1236999~$ &\raggedleft  $11~$ &\centering ---  &\centering --- &\centering --- &\centering ---  &\multicolumn{1}{r|}{$~~432.367~$}  \\
\hline
\end{tabular}
\vspace*{-3mm}
\caption{Experiments for the Modified LUBM}
\end{center}
\end{table}

\vspace*{-6mm}
\begin{table}[!h]
\setlength{\belowcaptionskip}{0.01pt}
\vspace*{-1.5mm}
\tabcolsep 0pt
\begin{center}
\tiny
\def\temptablewidth{1.02\columnwidth}
\begin{tabular}{|p{1.15cm}|m{1.1cm}|m{0.7cm}|m{1.1cm}|m{1.1cm}|m{1.1cm}|m{1.1cm}|m{1.1cm}|}
\hline
\cline{2-7}  {\sf ~Instance id} &
\multicolumn{1}{c|}{{\sf \#rules}~} & \multicolumn{1}{c|}{~ {\sf \#negs}~} & \multicolumn{1}{c|}{{\sf $t_\subseteq$}} &
\multicolumn{1}{c|}{{\sf $t_\leq$}~} & \multicolumn{1}{c|}{{\sf $t_{\subseteq_P}$}~} & \multicolumn{1}{c|}{{\sf $t_{\leq_P}$}} &
\multicolumn{1}{c|}{{\sf $t_{\leq_w}$}~}  \\
 \hline
 \hline
$~{\sf c1t1}$~&\raggedleft $170~$ &\raggedleft  $9~$ &\raggedleft $470.066~$ &\raggedleft $9.284~$ &\raggedleft $0.930~$ &\raggedleft $0.916~$  &\multicolumn{1}{r|}{$~~0.335~$}  \\
$~{\sf c1t3}$~&\raggedleft $170~$ &\raggedleft  $10~$ &\raggedleft $909.089~$ &\raggedleft $723.245~$ &\raggedleft $7.057~$ &\raggedleft $6.556~$  &\multicolumn{1}{r|}{$~~4.336~$}  \\
$~{\sf c1t5}$~&\raggedleft $170~$ &\raggedleft  $12~$ &\raggedleft $911.150~$ &\raggedleft $735.238~$&\raggedleft $28.906~$ &\raggedleft $28.344~$  &\multicolumn{1}{r|}{$~~12.284~$}  \\
$~{\sf c2t1}$~&\raggedleft $253~$ &\raggedleft  $9~$ &\raggedleft $1155.216~$ &\raggedleft $19.435~$ &\raggedleft $8.207~$ &\raggedleft $7.609~$  &\multicolumn{1}{r|}{$~~0.842~$}  \\
$~{\sf c2t3}$~&\raggedleft $253~$ &\raggedleft  $10~$ &\raggedleft $1171.904~$ &\raggedleft $1282.573~$ &\raggedleft $32.750~$ &\raggedleft $32.766~$  &\multicolumn{1}{r|}{$~~49.773~$}  \\
$~{\sf c2t5}$~&\raggedleft $253~$ &\raggedleft  $12~$ &\raggedleft $1136.926~$ &\raggedleft $1253.325~$ &\raggedleft $127.786~$ &\raggedleft $131.403~$  &\multicolumn{1}{r|}{$~~169.404~$}  \\
$~{\sf c3t1}$~&\raggedleft $361~$ &\raggedleft  $9~$ &\centering --- &\centering --- &\raggedleft $1423.179~$ &\raggedleft $1291.433~$  &\multicolumn{1}{r|}{$~~35.421~$}  \\
$~{\sf c3t3}$~&\raggedleft $361~$ &\raggedleft  $10~$ &\centering ---  &\centering --- &\centering --- &\centering ---  &\multicolumn{1}{c|}{---}  \\
\hline
\end{tabular}
\vspace*{-3mm}
\caption{Experiments for the Modified ChEBI}
\end{center}
\end{table}
\vspace*{-9.5mm}

\paragraph{Benchmarks}
To estimate
the performance of {\sf QAIER} in a view of data complexity,
we use the modified LUBM\footnote{$\textrm{LUBM}$.~http://swat.cse.lehigh.edu/projects/lubm/}
as a benchmark.
Because LUBM is not R-acyclic, we modified LUBM by changing atoms and deleting rules to make sure that modified LUBM is R-acyclic.
We use {\sf HermiT} \footnote{$\textrm{HermiT}$.~http://www.hermit-reasoner.com/} to transform the modified LUBM ontology into DL-clauses,
and replace at-least number restrictions in head atoms with existential
quantification,
then get 127 rules.
Next we add default negations or constraints, and introduce the prioritization and weight under rule {repair semantics}.
Considering that the number of default negations or constraints would not be very large,
we introduce 9-11 for each instance.
The introduced prioritization or weight depends on the reliability of the rules.
We use the EUDG\footnote{$\textrm{EUDG}$.http://www.informatik.uni-bremen.de/\!$\sim$clu/combined/} to generate a database.
{By $dXtY$ (Table 2) we mean that the instance involves $X$ thousands facts and $Y$ unreliable rules.}
For the performance in the view of combined complexity,
we use the modified ChEBI \cite{ijcai2013Magka} as a benchmark.
{By $cXtY$ (Table 3) we mean that the instance involves $X$ molecules and chemical classes and $Y$ unreliable rules.}
\vspace*{-5mm}
\paragraph{Experimental results}
Table 2 (Table 3, \!\!respectively)\footnote{All experiments run in Linux Ubuntu 14.04.1 LTS on a HP compaq 8200 elite with a 3.4GHz Intel Core i7 processor and 4G 1333 MHz memory. Real numbers in the tables figure the run time (in seconds) of query answering. If the time exceeds 1800 seconds, we write it as ``--".~~$\#facts$, $\#negs$, and $\#rules$ means the number of facts in database, default negations and constraints, and rules respectively.}
shows the data (combined, respectively) complexity performance among rule repairs scale up, when $\#facts$ and $\#negs$ ($\#rules$ and $\#negs$, respectively) grow.
$t_\subseteq$, $t_\leq$, $t_{\subseteq_P}$, $t_{\leq_P}$, or $t_{\leq_w}$
records the queries answering time. Each instance is computed three times and taken the average.
Because {\sf QAIER} computes all the stable models,
the sizes or the types of queries
are not the important issues.
Clearly, rule repairs \wrt $\subseteq_P$, $\leq_P$, and $\leq_w$ have better performances than those of $\subseteq$ and $\leq$, which is due to the few number of unreliable rules.
This condition can be easily found in realistic cases because most of the rules are reliable, while the latest learned rules considered unreliable are few.

\vspace*{-0.5mm}
\section{Related Work and Conclusions}
\vspace*{-0.5mm}
\looseness=-1
In terms of changing the rule set/Tbox for repair,
\citeauthor{MeyerLBP06} \shortcite{MeyerLBP06} proposed an algorithm running in \textsc{ExpTime} that finds maximally concept-satisfiable terminologies in $\mathcal{ALC}$.
\citeauthor{dlog2010Scharrenbach} \shortcite{dlog2010Scharrenbach} showed that probabilistic description logics can be used to resolve conflicts and receive a consistent
knowledge base from which inferences can be drawn again.
Also \citeauthor{ijcai2009QiD} \shortcite{ijcai2009QiD} proposed model-based revision operators for terminologies in DL, and
\citeauthor{dlog2014Wang} \shortcite{dlog2014Wang}
introduced a model-theoretic approach to ontology revision.
In order to address uncertainty arising from inconsistency,
\citeauthor{amai2013Gottlob} \shortcite{amai2013Gottlob}
extended the Datalog$^\pm$ language with probabilistic uncertainty
based on Markov logic networks.
More generally,
several works have focused on reasoning with inconsistent ontologies,
see \cite{ijcai2005Huang,semweb2005Haase} and references therein.
Surprisingly, this paper shows that for R-acyclic existential rules with R-stratified or guarded existential rules with stratified negations
both the data complexity and combined complexity of query answering under the rule {repair semantics} do not increase.

\looseness=-2
We have developed a general framework {to
handle inconsistent existential rules with default negations.}
Within this framework,
we analyzed the data and combined complexity of
inconsistency-tolerant query answering
under rule {repair semantics}.
We proposed approaches simulating queries answering under rule repairs with calling ASP solvers
and developed a prototype system called {\sf QAIER}.
Our experiments show that {\sf QAIER} can scale up to
large databases under rule repairs in practice.
Future work will focus on identifying first order rewritable classes under rule {repair semantics}.

\vspace{-.1cm}
\section{Acknowledgments}
\vspace{-.03cm}
\looseness=-1
We thank the reviewers for their comments and suggestions for improving the paper.
The authors would like to thank Yongmei Liu and her research group for their helpful and informative discussions. {
Hai Wan's research was in part supported by
the National Natural Science Foundation of China under grant 61573386, Natural Science Foundation of Guangdong Province of China under grant S2012010009836,
and Guangzhou Science and Technology Project (No. 2013J4100058).}

\small
\bibliographystyle{aaai}

\begin{thebibliography}{}

\end{thebibliography}


\begin{thebibliography}{}

\bibitem[\protect\citeauthoryear{Alviano and Pieris}{2015}]{AlvianoP15}
Alviano, M., and Pieris, A.
\newblock 2015.
\newblock Default negation for non-guarded existential rules.
\newblock In {\em Proceedings of the 34th {ACM} Symposium on Principles of
  Database Systems, {PODS} 2015, Melbourne, Australia, May 31 - June 4, 2015},
  79--90.

\bibitem[\protect\citeauthoryear{Arenas, Bertossi, and
  Chomicki}{1999}]{pods1999Arenas}
Arenas, M.; Bertossi, L.~E.; and Chomicki, J.
\newblock 1999.
\newblock Consistent query answers in inconsistent databases.
\newblock In {\em Proceedings of the Eighteenth {ACM} {SIGACT-SIGMOD-SIGART}
  Symposium on Principles of Database Systems, May 31 - June 2, 1999,
  Philadelphia, Pennsylvania, {USA}},  68--79.

\bibitem[\protect\citeauthoryear{Bienvenu, Bourgaux, and
  Goasdou{\'{e}}}{2014}]{aaai14BienvenuBG}
Bienvenu, M.; Bourgaux, C.; and Goasdou{\'{e}}, F.
\newblock 2014.
\newblock Querying inconsistent description logic knowledge bases under
  preferred repair semantics.
\newblock In {\em Proceedings of the Twenty-Eighth {AAAI} Conference on
  Artificial Intelligence, July 27 -31, 2014, Qu{\'{e}}bec City, Qu{\'{e}}bec,
  Canada.},  996--1002.

\bibitem[\protect\citeauthoryear{Cal{\`{\i}}, Gottlob, and
  Lukasiewicz}{2012}]{ws2012Cali}
Cal{\`{\i}}, A.; Gottlob, G.; and Lukasiewicz, T.
\newblock 2012.
\newblock A general datalog-based framework for tractable query answering over
  ontologies.
\newblock {\em Journal Web Semantics} 14:57--83.

\bibitem[\protect\citeauthoryear{Chang and Kadin}{1996}]{siamcomp1996Chang96}
Chang, R., and Kadin, J.
\newblock 1996.
\newblock The boolean hierarchy and the polynomial hierarchy: {A} closer
  connection.
\newblock {\em SIAM Journal on Computing} 25(2):340--354.

\bibitem[\protect\citeauthoryear{Chomicki}{2007}]{icdt2007Chomic}
Chomicki, J.
\newblock 2007.
\newblock Consistent query answering: Five easy pieces.
\newblock In {\em Proceedings of 11th International Conference, Database Theory
  - {ICDT} 2007, Barcelona, Spain, January 10-12, 2007,},  1--17.

\bibitem[\protect\citeauthoryear{Dantsin \bgroup et al\mbox.\egroup
  }{2001}]{csur2001DantsinEGV}
Dantsin, E.; Eiter, T.; Gottlob, G.; and Voronkov, A.
\newblock 2001.
\newblock Complexity and expressive power of logic programming.
\newblock {\em ACM Computing Surveys} 33(3):374--425.

\bibitem[\protect\citeauthoryear{Du, Qi, and Shen}{2013}]{kais2013Du}
Du, J.; Qi, G.; and Shen, Y.
\newblock 2013.
\newblock Weight-based consistent query answering over inconsistent
  $\mathcal{SHIQ}$ knowledge bases.
\newblock {\em Knowledge Information System} 34(2):335--371.

\bibitem[\protect\citeauthoryear{Eiter and Gottlob}{1995}]{jacm1995Eiter}
Eiter, T., and Gottlob, G.
\newblock 1995.
\newblock The complexity of logic-based abduction.
\newblock {\em Journal of the ACM} 42(1):3--42.

\bibitem[\protect\citeauthoryear{Ferraris, Lee, and
  Lifschitz}{2011}]{FerrarisLL11}
Ferraris, P.; Lee, J.; and Lifschitz, V.
\newblock 2011.
\newblock Stable models and circumscription.
\newblock {\em Artifical Intelligence} 175(1):236--263.

\bibitem[\protect\citeauthoryear{Gelfond and Lifschitz}{1988}]{GelfondL88}
Gelfond, M., and Lifschitz, V.
\newblock 1988.
\newblock The stable model semantics for logic programming.
\newblock In {\em Proceedings of the Fifth International Conference and
  Symposium Logic Programming, Seattle, Washington, August 15-19, 1988 {(2}
  Volumes)},  1070--1080.

\bibitem[\protect\citeauthoryear{Gottlob \bgroup et al\mbox.\egroup
  }{2013}]{amai2013Gottlob}
Gottlob, G.; Lukasiewicz, T.; Martinez, M.~V.; and Simari, G.~I.
\newblock 2013.
\newblock Query answering under probabilistic uncertainty in datalog+/-
  ontologies.
\newblock {\em Annals of Mathematics and Artificial Intelligence} 69(1):37--72.

\bibitem[\protect\citeauthoryear{Gottlob \bgroup et al\mbox.\egroup
  }{2014}]{kr2014Gottlob}
Gottlob, G.; Hernich, A.; Kupke, C.; and Lukasiewicz, T.
\newblock 2014.
\newblock Stable model semantics for guarded existential rules and description
  logics.
\newblock In {\em Proceedings of the Fourteenth International Conference
  Principles of Knowledge Representation and Reasoning, {KR} 2014, Vienna,
  Austria, July 20-24, 2014},  258--267.

\bibitem[\protect\citeauthoryear{Haase \bgroup et al\mbox.\egroup
  }{2005}]{semweb2005Haase}
Haase, P.; van Harmelen, F.; Huang, Z.; Stuckenschmidt, H.; and Sure, Y.
\newblock 2005.
\newblock A framework for handling inconsistency in changing ontologies.
\newblock In {\em Proceedings of The Semantic Web - {ISWC} 2005, 4th
  International Semantic Web Conference, {ISWC} 2005, Ireland, November 6-10,
  2005},  353--367.

\bibitem[\protect\citeauthoryear{Huang, van Harmelen, and ten
  Teije}{2005}]{ijcai2005Huang}
Huang, Z.; van Harmelen, F.; and ten Teije, A.
\newblock 2005.
\newblock Reasoning with inconsistent ontologies.
\newblock In {\em Proceedings of the Nineteenth International Joint Conference
  on Artificial Intelligence, {IJCAI} 2005, Edinburgh, Scotland, UK, July
  30-August 5, 2005},  454--459.

\bibitem[\protect\citeauthoryear{Kalyanpur \bgroup et al\mbox.\egroup
  }{2006}]{KalyanpurPSG06}
Kalyanpur, A.; Parsia, B.; Sirin, E.; and Grau, B.~C.
\newblock 2006.
\newblock Repairing unsatisfiable concepts in {OWL} ontologies.
\newblock In {\em Proceedings of the Semantic Web: Research and Applications,
  3rd European Semantic Web Conference, {ESWC} 2006, Budva, Montenegro, June
  11-14, 2006,},  170--184.

\bibitem[\protect\citeauthoryear{Lehmann \bgroup et al\mbox.\egroup
  }{2011}]{ws2011Lehmann}
Lehmann, J.; Auer, S.; B{\"{u}}hmann, L.; and Tramp, S.
\newblock 2011.
\newblock Class expression learning for ontology engineering.
\newblock {\em Journal Web Semantics} 9(1):71--81.

\bibitem[\protect\citeauthoryear{Lembo \bgroup et al\mbox.\egroup
  }{2010}]{rr2010LemboLRRS}
Lembo, D.; Lenzerini, M.; Rosati, R.; Ruzzi, M.; and Savo, D.~F.
\newblock 2010.
\newblock Inconsistency-tolerant semantics for description logics.
\newblock In {\em Proceedings of Web Reasoning and Rule Systems - Fourth
  International Conference, {RR} 2010, Bressanone/Brixen, Italy, September
  22-24, 2010.},  103--117.

\bibitem[\protect\citeauthoryear{Lukasiewicz \bgroup et al\mbox.\egroup
  }{2015}]{aaai2015Lukasiewicz}
Lukasiewicz, T.; Martinez, M.~V.; Pieris, A.; and Simari, G.~I.
\newblock 2015.
\newblock From classical to consistent query answering under existential rules.
\newblock In {\em Proceedings of the Twenty-Ninth {AAAI} Conference on
  Artificial Intelligence, January 25-30, 2015, Austin, {USA.}}

\bibitem[\protect\citeauthoryear{Lukasiewicz, Martinez, and
  Simari}{2012}]{ecai2012Lukasiewicz}
Lukasiewicz, T.; Martinez, M.~V.; and Simari, G.~I.
\newblock 2012.
\newblock Inconsistency handling in datalog+/- ontologies.
\newblock In {\em Proceedings of 20th European Conference on Artificial
  Intelligence. {ECAI} 2012 Including Prestigious Applications of Artificial
  Intelligence {(PAIS-2012)} System Demonstrations Track, Montpellier, France,
  August 27-31 , 2012},  558--563.

\bibitem[\protect\citeauthoryear{Lukasiewicz, Martinez, and
  Simari}{2013}]{otm2013Lukasiewicz}
Lukasiewicz, T.; Martinez, M.~V.; and Simari, G.~I.
\newblock 2013.
\newblock Complexity of inconsistency-tolerant query answering in datalog+/-.
\newblock In {\em Informal Proceedings of the 26th International Workshop on
  Description Logics, Ulm, Germany, July 23 - 26, 2013},  488--500.

\bibitem[\protect\citeauthoryear{Magka, Kr{\"o}tzsch, and
  Horrocks}{2013}]{ijcai2013Magka}
Magka, D.; Kr{\"o}tzsch, M.; and Horrocks, I.
\newblock 2013.
\newblock Computing stable models for nonmonotonic existential rules.
\newblock In {\em Proceedings of the 23rd International Joint Conference on
  Artificial Intelligence, {IJCAI} 2013, Beijing, China, August 3-9, 2013},
  1031--1038.

\bibitem[\protect\citeauthoryear{Meyer \bgroup et al\mbox.\egroup
  }{2006}]{MeyerLBP06}
Meyer, T.~A.; Lee, K.; Booth, R.; and Pan, J.~Z.
\newblock 2006.
\newblock Finding maximally satisfiable terminologies for the description logic
  {ALC}.
\newblock In {\em Proceedings of the Twenty-First National Conference on
  Artificial Intelligence and the Eighteenth Innovative Applications of
  Artificial Intelligence Conference, July 16-20, 2006, Boston, Massachusetts,
  {USA}},  269--274.

\bibitem[\protect\citeauthoryear{Qi and Du}{2009}]{ijcai2009QiD}
Qi, G., and Du, J.
\newblock 2009.
\newblock Model-based revision operators for terminologies in description
  logics.
\newblock In {\em Proceedings of the 21st International Joint Conference on
  Artificial Intelligence {IJCAI} 2009, Pasadena, California, USA, July 11-17,
  2009},  891--897.

\bibitem[\protect\citeauthoryear{Scharrenbach \bgroup et al\mbox.\egroup
  }{2010}]{dlog2010Scharrenbach}
Scharrenbach, T.; Gr{\"{u}}tter, R.; Waldvogel, B.; and Bernstein, A.
\newblock 2010.
\newblock Structure preserving tbox repair using defaults.
\newblock In {\em Proceedings of the 23rd International Workshop on Description
  Logics {(DL} 2010), Waterloo, Ontario, Canada, May 4-7, 2010},  384--395.

\bibitem[\protect\citeauthoryear{Wang \bgroup et al\mbox.\egroup
  }{2014}]{dlog2014Wang}
Wang, Z.; Wang, K.; Qi, G.; Zhuang, Z.; and Li, Y.
\newblock 2014.
\newblock Instance-driven tbox revision in dl-lite.
\newblock In {\em Informal Proceedings of the 27th International Workshop on
  Description Logics, Vienna, Austria, July 17-20, 2014.},  734--745.

\bibitem[\protect\citeauthoryear{Zhang, Zhang, and You}{2015}]{aaai2015heng}
Zhang, H.; Zhang, Y.; and You, J.-H.
\newblock 2015.
\newblock Existential rule languages with finite chase: Complexity and
  expressiveness.
\newblock In {\em Proceedings of the Twenty-Ninth {AAAI} Conference on
  Artificial Intelligence, January 25-30, 2015, Austin, Texas, {USA.}}

\end{thebibliography}

\end{document}